\documentclass[11pt]{article}
\usepackage{fullpage}

\usepackage{hyperref}
\hypersetup{pdftex,colorlinks=true,allcolors=blue} 
\usepackage{hypcap} 

\usepackage[utf8]{inputenc} 
\usepackage[T1]{fontenc}    
\usepackage{url}            
\usepackage{booktabs}       
\usepackage{amsfonts}       
\usepackage{microtype}      

\usepackage{amsmath, amsthm, amssymb}
\usepackage{graphicx}
\usepackage{cite}
\usepackage{mathrsfs}
\usepackage{enumitem} 
\usepackage{lmodern}

\newtheorem{theorem}{Theorem}
\newtheorem{lemma}{Lemma}

\newtheorem{corollary}{Corollary}
\newtheorem{definition}{Definition}
\newcommand{\R}{\mathbb{R}}
\newcommand{\E}{\mathbb{E}}

\newcommand{\I}{\mathbf{1}}

\DeclareMathOperator*{\argmin}{arg\,min}

\def\sU{{\mathsf U}}
\def\sV{{\mathsf V}}

\def\sX{{\mathsf X}}
\def\sY{{\mathsf Y}}
\def\sZ{{\mathsf Z}}

\def\deq{\triangleq}
\def\wh#1{{\widehat{#1}}}

\newcommand\blfootnote[1]{%
	\begingroup
	\renewcommand\thefootnote{}\footnote{#1}%
	\addtocounter{footnote}{-1}%
	\endgroup
}

\title{Dynamic Inference}
\author{Aolin Xu} 
\date{}

\begin{document}
	
	\maketitle
	
	\begin{abstract}
		Traditional statistical estimation, or statistical inference in general, is static, in the sense that the estimate of the quantity of interest does not change the future evolution of the quantity. In some sequential estimation problems however, we encounter the situation where the future values of the quantity to be estimated depend on the estimate of its current value. Examples include stock price prediction by big investors, interactive product recommendation, and behavior prediction in multi-agent systems.
		We may call such problems as \emph{dynamic inference}. In this work, a formulation of this problem under a Bayesian probabilistic framework is given, and the optimal estimation strategy is derived as the solution to minimize the overall inference loss. 
		How the optimal estimation strategy works is illustrated through two examples, stock trend prediction and vehicle behavior prediction.
		When the underlying models for dynamic inference are unknown, we can consider the problem of \emph{learning for dynamic inference}. This learning problem can potentially unify several familiar machine learning problems, including supervised learning, imitation learning, and reinforcement learning.
	\end{abstract}

	\section{Introduction}
	Traditional statistical estimation, or statistical inference in general, is static, in the sense that the estimate of the quantity of interest does not affect the future evolution of the quantity. 
	In some sequential estimation problems however, we encounter the situation where the future values of the quantity to be estimated depend on the estimate of its current value. 
	Examples include: 1) stock price prediction by big investors, where the prediction of today's price of a stock affects today's investment decision, which further changes the stock's supply-demand status and hence its price tomorrow; 2) interactive product recommendation, where the estimate of a customer's preference based on their activity leads to certain product recommendations, which would in turn shape the customer's future activity and preference; 3) behavior prediction in multi-agent systems, e.g.\ predicting the intentions of vehicles on the road adjacent to the ego vehicle, where the prediction of an adjacent vehicle's intention based on its current driving situation leads to a certain action of the ego vehicle, which can change the future driving situation and intention of that adjacent vehicle.

	Broadly speaking, this type of interactive sequential estimation problems arises in any autonomous agent that interacts with a system of interest, through a measurement-inference-action loop. During the interaction, the inference, either estimation or prediction, of a property of the system based on the measurements of its current and past states affects the action to be taken by the autonomous agent, which further influences the future states and properties of the system of interest.
	We may call such problems as \emph{dynamic inference}.

	In Section~\ref{sec:formulation}, a mathematical formulation of dynamic inference is given under a Bayesian probabilistic framework. 
	It is shown in Section~\ref{sec:MDP_sol} that this problem can be converted to a Markov decision-making process (MDP), and the optimal estimation strategy that minimizes the overall inference loss can be derived as the optimal policy of this MDP through dynamic programming. 
	Two examples, stock trend prediction and vehicle behavior prediction, are given in Section~\ref{sec:app} to illustrate how the optimal estimation strategy for dynamic inference works, and how it differs from the solution to the traditional statistical inference.
	
	Section~\ref{sec:learn_di} briefly discusses the problem of \emph{learning for dynamic inference}, which is to address the situation where the underlying probabilistic models for dynamic inference become unknown. Learning for dynamic inference can potentially serve as a unifying meta problem of machine learning, such that supervised learning, imitation learning, and reinforcement learning can be cast as its special instances. 
	Having a good understanding of dynamic inference and its learning extension will thus be helpful in gaining better understandings of a broad spectrum of machine learning problems.
	The formulation of dynamic inference appears to be new, but it can be related to a variety of existing interactive decision-making problems and prediction problems that take the consequence of the prediction into account. Moreover, any MDP may be thought of as a dynamic inference problem. These related problems are discussed in Section~\ref{sec:rel_work}.

	\section{Problem formulation}\label{sec:formulation}
	\subsection{Traditional statistical inference}
	In traditional statistical inference, the goal is to estimate a quantity of interest $Y$ based on an observation $X$ that statistically depends on $Y$. 
	Under the Bayesian formulation, the pair $(X,Y)\in\sX\times\sY$ is modeled as a jointly distributed random vector with distribution $P_{X,Y}$.
	Given a loss function $\ell:\sY\times\wh\sY\rightarrow\R$, the optimal estimator $\psi_{\rm B}$, a.k.a.\ the Bayes estimator, is a map $ \sX\rightarrow\wh\sY$ that achieves the minimum expected loss:
	\begin{align}\label{eq:Bayes_inf}
		\psi_{\rm B} = \argmin_{\psi:\sX\rightarrow\wh\sY} \E[\ell(Y,\psi(X))] .
	\end{align}
	A basic result from estimation theory is that for any $x\in\sX$, the optimal estimate of $Y$ given $X=x$ is a minimizer of the expected posterior loss, i.e.\ $\psi_{\rm B}(x)=\argmin_{\hat y\in\wh\sY}\E[\ell(Y,\hat y)|X=x]$.
	The above statistical inference problem is \emph{static}, in the sense that only one round of estimation is considered. 
	When there is a need to estimate a sequence of quantities $Y^n\deq (Y_1,\ldots,Y_n)$ based on observations $X^n \deq (X_1,\ldots,X_n)$, if the pairs $(X_i,Y_i)$ are i.i.d.\ for $i=1,\ldots,n$, the sequential estimation problem to minimize the accumulated expected loss can be optimally solved by repeatedly using the same single-round optimal estimator $\psi_{\rm B}$.
	
	\subsection{Dynamic inference}
	The problem of $n$-round dynamic inference is to {sequentially} estimate $n$ quantities of interest $Y^n$ based on observations $X^n$, where in each round, the quantity of interest $Y_i$ only depends on the observation $X_i$, while $X_i$ depends on the observation $X_{i-1}$ and the estimate $\wh Y_{i-1}$ of $Y_{i-1}$ in the previous round; the estimate $\wh Y_i$ of $Y_i$ is made potentially based on all the information available so far, namely $(X^{i},Y^{i-1})$, with the goal of minimizing the expected accumulated loss over the $n$ rounds. Here it is assumed that after the $i$th round of estimation, $Y_i$ is revealed to the estimator. It can also happen that $Y^n$ are never revealed during the process, in which case $\wh Y_i$ is estimated only based on $X^{i}$. Nevertheless, it will be shown in Section~\ref{sec:MDP_sol} that an optimal estimation strategy can estimate $Y_i$ only based on the instantaneous observation $X_i$, no matter $Y^{i-1}$ are available or not.
	
	Formally, we assume the knowledge of the distribution $P_{X_1}$ of the initial observation, the probability transition kernel $P{\raisebox{-2pt}{$\scriptstyle X_i|X_{i-1}, \wh Y_{i-1}$}}$ of the $i$th observation given the observation and the estimate in the previous round, $i=2,\ldots,n$, and the probability transition kernel $P_{Y_i|X_i}$ of the $i$th quantity of interest given the $i$th observation, $i=1,\ldots,n$. 
	We may call these two types of probability transition kernels the \emph{observation-transition model} and the \emph{quantity-generation model}, respectively.
	The estimates $\wh Y^n$ are sequentially made according to an \emph{estimation strategy}: 
	\begin{definition}
		An estimation strategy for an $n$-round dynamic inference is a sequence of estimators $\psi^n = (\psi_1,\ldots,\psi_n)$, where $\psi_i:\sX^{i}\times\sY^{i-1}\rightarrow\wh\sY$ is the estimator used in the $i$th round, $i=1,\ldots,n$, which maps the history of observations and revealed quantities of interest $(X^i,Y^{i-1})$ to an estimate $\wh Y_i$ of $Y_i$, such that $\wh Y_i = \psi_i(X^i,Y^{i-1})$.
	\end{definition}
	Any specification of $P_{X_1}$, $(P{\raisebox{-2pt}{$\scriptstyle X_i|X_{i-1}, \wh Y_{i-1}$}})_{i=2}^n$, $(P_{Y_i|X_i})_{i=1}^n$ and $\psi^n$ defines a joint distribution of all the random variables $(X^n,Y^n,\wh Y^n)$ under consideration.
	The Bayesian network of the random variables in dynamic inference with a Markov estimation strategy, meaning that each estimate has the form $\psi_i:\sX\rightarrow\wh\sY$, is illustrated in Fig.~\ref{fig:BN_DI}.
	\begin{figure}[h]
		\centering
		\includegraphics[scale = 0.6]{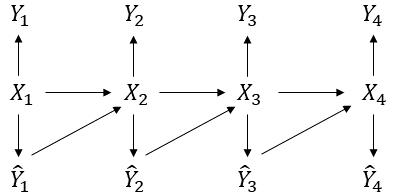}
		\caption{Bayesian network of the random variables under consideration with $n=4$. Here we assume the estimates are made with Markov estimators, such that $\wh Y_i = \psi_i(X_i)$.}
		\label{fig:BN_DI}
	\end{figure}
	
	We use a loss function $\ell:\sX\times\sY\times\wh\sY\rightarrow\R$ to evaluate the estimate made in each round in dynamic inference. This loss function is a generalization of the one used in statistical inference, in the sense that the estimate in each round is evaluated in the context of the observation in that round.
	Given an estimation strategy $\psi^n$, we define its \emph{inference loss} as the expected accumulated loss over the $n$ rounds,
	$\E\big[ \sum_{i=1}^n \ell(X_i, Y_i, \wh Y_i) \big]$.
	The goal of dynamic inference is to find an estimation strategy to minimize the inference loss:
	\begin{align}\label{eq:di_loss}
		\argmin_{\psi^n} \,\, \E\Big[ \sum_{i=1}^n \ell(X_i, Y_i, \wh Y_i) \Big] ,
	\end{align}
	where $\wh Y_i=\psi_i(X^i,Y^{i-1})$.
	Comparing with the statistical inference problem in \eqref{eq:Bayes_inf}, we summarize the two distinctive features of dynamic inference:
	\begin{itemize}[leftmargin=*]
		\item
		The joint distribution of the pair $(X_i,Y_i)$ changes in each round in a controlled manner, as it depends on $(X_{i-1},\wh Y_{i-1})$. 
		\item
		The loss in each round is contextual, as it depends on $X_i$.
	\end{itemize}

	\section{Optimal estimation strategy}\label{sec:MDP_sol}
	In this section we show that the dynamic inference problem can be converted to a Markov decision process, and the optimal estimation strategy can be found via dynamic programming.
	
	\subsection{MDP reformulation}
	\subsubsection{Equivalent expression of inference loss}
	For a given loss function $\ell:\sX\times\sY\times\wh\sY\rightarrow\R$ and a joint distribution of $(X,Y,\wh Y)$, we can define a corresponding \emph{observation-estimate loss function} $\bar \ell:\sX\times\wh\sY\rightarrow\R$ as
	\begin{align}\label{eq:def_bar_ell}
		\bar \ell(x,\hat y) \deq \E[\ell(x,Y,\hat y)|X=x, \wh Y=\hat y] 
	\end{align}
	for $(x,\hat y)\in \sX\times\wh Y $.
	From the specification of the joint distribution of the random variables in the previous section, we know that in dynamic inference $Y_i$ is conditionally independent of $\wh Y_i$ given $X_i$, therefore for any realization $(x_i,\hat y_i)$ of $(X_i,\wh Y_i)$, the value of the $i$th observation-estimate loss $\bar \ell(x_i,\hat y_i)$ can be computed as 
	\begin{align}\label{eq:def_bar_ell_comp}
		\bar \ell(x_i,\hat y_i) = \E[\ell(x_i,Y_i,\hat y_i)|X_i=x_i] .
	\end{align}
	We see that $\bar\ell$ as a function of $(x_i,\hat y_i)$ is determined by $\ell$ and $P_{Y_i|X_i}$, and does not depend on the estimator $\psi_i$. This fact is crucial for the optimality proof later.
	With the above definition, the inference loss can be expressed in terms of the observation-estimate loss: 
	\begin{lemma}\label{lm:acc_loss_di}
		For any estimation strategy, the inference loss in \eqref{eq:di_loss} can be rewritten as
		\begin{align}
			\E\Big[ \sum_{i=1}^n \ell(X_i, Y_i, \wh Y_i) \Big] =  \E\Big[ \sum_{i=1}^n \bar \ell(X_i, \wh Y_i) \Big] .
		\end{align}
	\end{lemma}
	\begin{proof}
For each $i=1,\ldots,n$,
\begin{align}
	\E[ \ell(X_i, Y_i, \wh Y_i) ] &= \E\big[ \E[\ell( X_i, Y_i, \wh Y_i) | X^i, Y^{i-1}] \big] \label{eq:pf_acc_loss_di_1} \\
	&= \E\big[ \E[\ell( X_i, Y_i, \wh Y_i) | X_i, \wh Y_i] \big] \label{eq:pf_acc_loss_di_2} \\
	&=  \E[ \bar \ell(X_i, \wh Y_i) ] \label{eq:pf_acc_loss_di_3}
\end{align}
where \eqref{eq:pf_acc_loss_di_2} is a consequence of the fact that $\wh Y_i$ is determined by $(X^{i},Y^{i-1})$ and the fact that $Y_i$ is conditionally independent of $(X^{i-1}, Y^{i-1})$ given $X_i$; and \eqref{eq:pf_acc_loss_di_3} is due to the definition of $\bar\ell$.
The claim then follows from the fact that
$
\E\big[ \sum_{i=1}^n \ell(X_i, Y_i, \wh Y_i) \big] = \sum_{i=1}^n \E\big[ \ell(X_i, Y_i, \wh Y_i) \big] .
$
	\end{proof}
	
	With Lemma~\ref{lm:acc_loss_di}, the optimization problem in \eqref{eq:di_loss} becomes equivalent to
	\begin{align}\label{eq:di_min_loss}
		\argmin_{\psi^n} \,\, J(\psi_n) ,
	\end{align}
	where $J(\psi_n) \deq \E\big[ \sum_{i=1}^n \bar \ell(X_i, \wh Y_i)\big ]$ equals to the inference loss, and $\wh Y_i=\psi_i(X^i,Y^{i-1})$.

	\subsubsection{Optimality of Markov estimators}
	Next, we show that the search space of the optimization problem in \eqref{eq:di_min_loss} can be restricted to Markov estimators $\bar\psi_i:\sX\rightarrow\wh\sY$, such that $\wh Y_i = \bar\psi_i(X_i)$.
	We start with a lemma known as Blackwell's principle of irrelevant information \cite{MR_SOC_notes}, and provide a proof for completeness.
	\begin{lemma}\label{lm:Blackwell}
		For any fixed function $f:\sU\times\sV\rightarrow \R$ and for any jointly distributed pair $(U,Z)$,
		\begin{align}\label{eq:Blackwell}
			\min_{g:\sU\times\sZ\rightarrow\sV} \E\big[f\big(U,g(U,Z)\big)\big] = \min_{g:\sU\rightarrow\sV} \E\big[f\big(U, g(U)\big)\big] .
		\end{align}
	\end{lemma}
	\begin{proof}
		The left side of \eqref{eq:Blackwell} is the Bayes risk of estimating $U$ based on $(U,Z)$, defined with respect to the loss function $f$, and can be written as $R_f(U|U,Z)$;  while the right side of \eqref{eq:Blackwell} is the Bayes risk of estimating $U$ based on $U$ itself, also defined with respect to the loss function $f$, and can be written as $R_f(U|U)$. 
		It is clear from their definitions that $R_\ell(U|U,Z) \le R_\ell(U|U)$. 
	It also follows from a data processing inequality of Bayes risk \cite[Lemma~1]{MER20} that
	\begin{align}
		R_\ell(U|U,Z) \ge R_\ell(U|U) ,
	\end{align}
	as $U-U-(U,Z)$ form a Markov chain. Hence $R_f(U|U,Z)=R_f(U|U)$, which proves the claim.
	\end{proof}

	The first application of Lemma~\ref{lm:Blackwell} is to prove that the last estimator of an optimal estimation strategy can be replaced by a Markov one, which preserves the optimality.
	\begin{lemma}[Last-round lemma]\label{lm:last_round} 
		Given any estimation strategy $\psi^n$, there exists a Markov estimator $\bar\psi_{n}:\sX\rightarrow\wh\sY$, such that
		$
		J(\psi_{1},\ldots,\psi_{n-1}, \bar\psi_{n}) \le J(\psi^{n}) .
		$
	\end{lemma}
	\begin{proof}
		According to Lemma~\ref{lm:acc_loss_di}, the inference loss of $\psi^n$ can be written as
		\begin{align}
			J(\psi^{n}) 
			&= \E\Big[\sum_{i=1}^{n-1} \bar{\ell}( X_i, \wh Y_i)\Big]
			+ \E\big[ \bar{\ell}(X_n, \psi_{n}(X^{n}, Y^{n-1}) ) \big] . \label{eq:pf_last_round_lm_2}
		\end{align}
		Since the first expectation in \eqref{eq:pf_last_round_lm_2} does not depend on $\psi_{n}$, it suffices to show that there exists a Markov estimator $\bar\psi_{n}:\sX\rightarrow\wh\sY$, such that 
		\begin{align}
			\E\big[ \bar{\ell}\big(X_n, \bar\psi_{n}(X_n) \big)\big]  \le
			\E\big[ \bar{\ell}\big(X_n, \psi_{n}(X^{n}, Y^{n-1}) \big)\big] .
		\end{align}
		The existence of such an estimator is guaranteed by Lemma~\ref{lm:Blackwell}.
	\end{proof}

	Lemma~\ref{lm:Blackwell} can be further used to prove that whenever the last estimator is Markov, its preceding estimator can also be replaced by a Markov one which preserves the optimality.
	\begin{lemma}[$(i-1)$th-round lemma]\label{lm:i-1th_round} 
		For any $i\ge 2$, given any estimation strategy $(\psi_{1},\ldots,\psi_{i-1}, \bar\psi_{i})$ for an $i$-round dynamic inference, if the last estimator is a Markov one $\bar \psi_{i}:\sX\rightarrow\wh\sY$, then there exists a Markov estimator $\bar\psi_{i-1}:\sX\rightarrow\wh\sY$ for the $(i-1)$th round, such that
		$
		J(\psi_{1},\ldots,\psi_{i-2}, \bar\psi_{i-1}, \bar\psi_{i}) \le J(\psi_{1},\ldots,\psi_{i-1}, \bar\psi_{i}) .
		$
	\end{lemma}
	\begin{proof}
	According to Lemma~\ref{lm:acc_loss_di}, the inference loss of the given $(\psi_{1},\ldots,\psi_{i-1}, \bar\psi_{i})$ is
	\begin{align}
		J(\psi_{1},\ldots,\psi_{i-1}, \bar\psi_{i})
		= \E\Big[\sum_{j=1}^{i-2} \bar{\ell}( X_j, \wh Y_j)\Big] + 
		\E\big[ \bar{\ell}(X_{i-1}, \wh Y_{i-1} )\big] + 
		\E\big[ \bar{\ell}(X_i, \bar\psi_{i}(X_i) \big)\big] . \label{eq:pf_i-1_round_lm_2}
	\end{align}
	Since the first expectation in \eqref{eq:pf_i-1_round_lm_2} does not depend on $\psi_{i-1}$, it suffices to show that there exists a Markov estimator $\bar\psi_{i-1}:\sX\rightarrow\wh\sY$, such that 
	\begin{align}
		\E\big[ \bar{\ell}(X_{i-1}, \bar\psi_{i-1}(X_{i-1}) \big)\big] + 
		\E\big[ \bar{\ell}(\bar X_i, \bar\psi_{i}(\bar X_i) \big)\big]
		\le 
		\E\big[ \bar{\ell}(X_{i-1}, \wh Y_{i-1} \big)\big] + 
		\E\big[ \bar{\ell}(X_i, \bar\psi_{i}(X_i) \big)\big] , \label{eq:pf_i-1_round_lm_3}
	\end{align}
	where $\bar X_i$ on the left side is the observation in the $i$th round when the Markov estimator $\bar\psi_{i-1}$ is used in the $(i-1)$th round.
	To get around with the dependence of $X_i$ on $\psi_{i-1}$, we write the second expectation on the right side of \eqref{eq:pf_i-1_round_lm_3} as 
	\begin{align}
		\E\big[ \E\big[ \bar{\ell}( X_i, \bar\psi_{i}(X_i) \big) \big| X_{i-1}, \wh Y_{i-1}\big] \big]
	\end{align}
	and notice that the inner conditional expectation as a function of $( X_{i-1}, \wh Y_{i-1})$ does not depend on $\psi_{i-1}$. This is because the conditional distribution of $X_i$ given $(X_{i-1}, \wh Y_{i-1})$ is specified by the probability transition kernel $P{\raisebox{-2pt}{$\scriptstyle X_i|X_{i-1}, \wh Y_{i-1}$}}$.
	It follows that the right side of \eqref{eq:pf_i-1_round_lm_3} can be written as
	\begin{align}
		& \E\Big[ \bar{\ell}\big(X_{i-1}, \wh Y_{i-1} \big) + 
		\E\big[ \bar{\ell}\big(X_i, \bar\psi_{i}(X_i) \big) \big|  X_{i-1}, \wh Y_{i-1}\big] \Big] & \nonumber \\
		=& \E\big[f\big(	X_{i-1}, \wh Y_{i-1} \big)\big] \\
		=& \E\big[f\big(X_{i-1}, \psi_{i-1}(X^{i-1}, Y^{i-2})\big) \big] ,
	\end{align}
	where the function $f$ does not depend on $\psi_{i-1}$.
	It follows from Lemma~\ref{lm:Blackwell} that there exists an estimator $\bar\psi_{i-1}:\sX\rightarrow\wh\sY$, such that 
	\begin{align}
		& \E\big[f\big(X_{i-1}, \psi_{i-1}(X^{i-1}, Y^{i-2})\big) \big] \nonumber \\
		\ge & \E\big[f\big(	X_{i-1}, \bar\psi_{i-1}( X_{i-1}) \big)\big] \\
		= & \E\Big[ \bar{\ell}( X_{i-1}, \bar\psi_{i-1}(X_{i-1}) \big) + 
		\E\big[ \bar{\ell}(\bar X_i, \bar\psi_{i}( \bar X_i) \big) \big| X_{i-1}, \bar\psi_{i-1}( X_{i-1})\big] \Big] \\
		= & \E\big[ \bar{\ell}( X_{i-1}, \bar\psi_{i-1}(X_{i-1}) \big)\big] + 
		\E\big[ \bar{\ell}( \bar X_i, \bar\psi_{i}( \bar X_i) \big)\big] ,
	\end{align}
	which proves \eqref{eq:pf_i-1_round_lm_3} and the claim.
	\end{proof}
	
	With Lemma~\ref{lm:last_round} and Lemma~\ref{lm:i-1th_round}, we can finally prove the optimality of Markov estimators.
	\begin{theorem}\label{th:Markov_DI}
		The minimum of $J(\psi^n)$ in \eqref{eq:di_min_loss}
		can be achieved by an estimation strategy $\bar\psi^n$ with Markov estimators $\bar\psi_{i}:\sX\rightarrow\wh\sY$, $i=1,\ldots,n$, such that $\wh Y_i = \bar\psi_{i}(X_i)$.
	\end{theorem}
	\begin{proof}
		Picking an optimal estimation strategy $\psi^n$, we first replace its last estimator by a Markov one that preserves the optimality of the strategy, as guaranteed by Lemma~\ref{lm:last_round}.
		Then, for $i=n,\ldots,2$, we repeatedly replace the $(i-1)$th estimator by a Markov one that preserves the optimality of the previous strategy, as guaranteed by Lemma~\ref{lm:i-1th_round} and the additive structure of the inference loss as in $J(\psi^n)$.
		Finally we obtain an estimation strategy consisting of Markov estimators achieving the same inference loss as the originally picked strategy.
	\end{proof}

	\subsubsection{Conversion to MDP}\label{sec:convert_mdp}
	Theorem~\ref{th:Markov_DI} with Lemma~\ref{lm:acc_loss_di} imply that the original dynamic inference problem in \eqref{eq:di_loss} is equivalent to
	\begin{align}\label{eq:di_MDP}
		\argmin_{ \psi^n} \E\Big[ \sum_{i=1}^n \bar{\ell}(X_i, \wh Y_i) \Big] , \quad \wh Y_i = \psi_{i}( X_i) .
	\end{align}
	With this reformulation, we see that the unknown quantities $Y_i$ do not appear in the loss function in \eqref{eq:di_MDP} any more, and the optimization problem becomes a standard MDP:
	the observations $X^n$ become the states in this MDP, the estimates $\wh Y^n$ become the actions, the probability transition kernel $P{\raisebox{-2pt}{$\scriptstyle X_i|X_{i-1}, \wh Y_{i-1}$}}$ now defines the controlled state transition, and any Markov estimation strategy $\psi^n$ becomes a policy of this MDP. 
	The goal of dynamic inference then becomes finding the optimal policy of this MDP to minimize the expected accumulated loss with respect to $\bar\ell$. 
	Conversely, the solution to the MDP will be an optimal estimation strategy for dynamic inference.

	\subsection{Solution via dynamic programming}\label{sec:sol_dp}
	\subsubsection{Optimal estimation strategy}
	From the theory of MDP \cite{dp_oc_vol1} it is known that the optimal policy for the MDP in \eqref{eq:di_MDP}, or the optimal estimation strategy for dynamic inference, can be found via dynamic programming.
	To derive the optimal estimators, define the functions $Q^*_i : \sX\times\wh Y\rightarrow\R$ and $V^*_i:\sX\rightarrow\R$ recursively in backward as 
	$
	Q^*_n(x,\hat y) \deq \bar\ell (x, \hat y) 
	$
	,
	\begin{align}
	V^*_i(x) \deq \min_{\hat y \in\wh \sY} Q^*_i (x, \hat y)   , \quad  i=n,\ldots,1, \label{eq:dp_V} 
	\end{align}
	and
	\begin{align}
		Q^*_i(x,\hat y) &\deq \bar\ell (x, \hat y) + \E[V^*_{i+1}(X_{i+1})|X_i=x, \wh Y_i=\hat y], \quad i=n-1,\ldots,1 \label{eq:dp_Q} .
	\end{align}
	The optimal estimate to make in the $i$th round when $X_i=x$ is then
	\begin{align}\label{eq:dp_policy}
		\psi_i^*(x) \deq \argmin_{\hat y \in\wh \sY} Q^*_i (x, \hat y) , \quad i=1,\ldots,n .
	\end{align}

	\subsubsection{Minimum inference loss and loss-to-go}
	For any estimation strategy $\psi^n$, we can define its loss-to-go at the $i$th round of estimation when $X_i=x$ as
	the conditional expected loss accumulated from the $i$th round to the final round given that the observation in the $i$th round is $x$:
	\begin{align}\label{eq:def_ctg}
		V_i(x; \psi^n) \deq \E\Big[ \sum_{j=i}^n \ell(X_j, Y_j, \wh Y_j) \Big| X_i=x\Big] .
	\end{align}
	The following theorem
	states that the estimation strategy $(\psi^*_1,\ldots,\psi^*_n)$ derived from dynamic programming not only achieves the minimum inference loss, but also achieves the minimum loss-to-go in each round with any observation in that round.
	\begin{theorem}\label{th:opt_dp}
		The estimators $(\psi^*_1,\ldots,\psi^*_n)$ defined in \eqref{eq:dp_policy} 
		constitute an optimal estimation strategy for dynamic inference, which achieves the minimum in \eqref{eq:di_loss}.
		Moreover, for any Markov estimation strategy $\psi^n$ with $\psi_i:\sX\rightarrow\wh\sY$, its loss-to-go satisfies
		\begin{align}
			V_i(x; \psi^n) \ge V^*_i(x) \label{eq:opt_V_dp}
		\end{align}
		for all $x\in\sX$ and $i=1,\ldots,n$,
		with equality if $\psi_j(x)=\psi_j^*(x)$ for all $x\in\sX$ and $j\ge i$.
	\end{theorem}
	\begin{proof}
	The first claim stating that the estimation strategy $(\psi^*_1,\ldots,\psi^*_n)$ achieves the minimum in \eqref{eq:di_loss} follows from the equivalence between the original problem in \eqref{eq:di_loss} and the MDP reformulation in \eqref{eq:di_MDP} as discussed in Section~\ref{sec:convert_mdp}, and from the well-known optimality of the solution via dynamic programming to MDP \cite{dp_oc_vol1}.
	
	The second claim can be proved via backward induction.
	Consider an arbitrary Markov estimation strategy $\psi^n$.
	\begin{itemize}[leftmargin=*]
		\item 
		In the final round, for all $x\in\sX$,
		\begin{align}
			V_n(x; \psi^n) 
			&= \E[ \ell(x,Y_n,\psi_n(x))|X_n = x ] \label{eq:pf_opt_dp1} \\
			&= \bar \ell(x,\psi_n(x)) \label{eq:pf_opt_dp2} \\
			&\ge V^*_n(x) , \label{eq:pf_opt_dp3} 
		\end{align}
		where \eqref{eq:pf_opt_dp1} follows from the definition of $V_n$ in \eqref{eq:def_ctg}; \eqref{eq:pf_opt_dp2} follows from the way how $\bar\ell$ can be computed as in \eqref{eq:def_bar_ell_comp}; and \eqref{eq:pf_opt_dp3} follows from the definition of $V^*_n$ above \eqref{eq:dp_Q}, while the equality holds if $\psi_n(x)=\psi^*_n(x)$.
		\item
		For $i=n-1,\ldots,1$, as the inductive assumption, suppose \eqref{eq:opt_V_dp} holds in the $(i+1)$th round. 
		We first show a self-recursive expression of $V_i(x; \psi^n)$:
		\begin{align}
			V_i(x; \psi^n) 
			&= \E\Big[ \sum_{j=i}^n \ell(X_j, Y_j, \wh Y_j) \Big| X_i=x\Big] \\
			&= \E[\ell(X_i,Y_i,\wh Y_i)|X_i=x] + 
			\E\Big[ \sum_{j=i+1}^n \ell(X_j, Y_j, \wh Y_j) \Big| X_i=x\Big] \\
			&= \E\big[\E[\ell(X_i,Y_i,\wh Y_i)|X_i=x, \wh Y_i] \big|X_i=x\big] +  \nonumber \\
			& \quad \,\,
			\E\Big[\E\Big[ \sum_{j=i+1}^n \ell(X_j, Y_j, \wh Y_j) \Big| X_i=x, X_{i+1} \Big] \Big |  X_i=x \Big] \\
			&= \E\big[\bar\ell(x,\wh Y_i) \big|X_i=x\big] + 
			\E\Big[\E\Big[ \sum_{j=i+1}^n \ell(X_j, Y_j, \wh Y_j) \Big|X_{i+1}\Big] \Big |  X_i=x \Big] \label{eq:pf_V_recur_1} \\
			&= \bar\ell(x,\psi_i(x)) + \E\big[V_i(X_{i+1};\psi^n) |  X_i=x \big] 
		\end{align}
		where the first term of \eqref{eq:pf_V_recur_1} follows from the definition of $\bar\ell$ in \eqref{eq:def_bar_ell}, while the second term of \eqref{eq:pf_V_recur_1} follows from the fact that $X_i$ is conditionally independent of $(X_{i+1}^n, Y_{i+1}^n, \wh Y_{i+1}^n)$ given $X_{i+1}$, which is a consequence of the assumption that the estimators under consideration are Markov and the specification of the joint distribution of $(X^n, Y^n, \wh Y^n)$ in Section~\ref{sec:formulation}.
		Then,
		\begin{align}
			V_i(x; \psi^n) 
			&\ge \bar\ell(x,\psi_i(x)) + 
			\E\big[V^*_i(X_{i+1}) |  X_i=x \big] \label{eq:pf_opt_V_dp_1} \\
			&= \bar\ell(x,\psi_i(x)) + 
			\E\big[V^*_i(X_{i+1}) |  X_i=x, \wh Y_i = \psi_i(x) \big] \label{eq:pf_opt_V_dp_2} \\
			&= Q^*_i(x,\psi_i(x)) \label{eq:pf_opt_V_dp_3} \\
			&\ge V^*_i(x)
		\end{align}
		where \eqref{eq:pf_opt_V_dp_1} follows from the inductive assumption; \eqref{eq:pf_opt_V_dp_2} follows from the fact that $\wh Y_i$ is determined given $X_i=x$ through the Markov estimator $\psi_i$; \eqref{eq:pf_opt_V_dp_3} follows from the definition of $Q^*_i$ in \eqref{eq:dp_Q}; and the final equality condition follows from the definitions of $V^*_i$ above \eqref{eq:dp_Q} and $\psi^*_i$ in \eqref{eq:dp_policy}.
	\end{itemize}
	This proves the second claim.
	\end{proof}
	
	A consequence of Theorem~\ref{th:opt_dp} is that the minimum loss-to-go at the $i$th round can be expressed in terms of $V^*_i$.
	Moreover, once the values of $V^*_i(x)$ for all $x\in\sX$ and $i=1,\ldots,n$ are computed, the optimal estimation strategy for any $(n-i+1)$-round dynamic inference with the same model for $j=i,\ldots,n$ can be determined by these values and the observation-estimate loss function $\bar\ell$.
	These results are stated in the following corollary.
	\begin{corollary}
		For any $i=1,\ldots,n$ and any initial distribution $P_{X_i}$,
		\begin{align}
			\min_{\psi_i,\ldots,\psi_n} \,\, \E\Big[ \sum_{j=i}^n \ell(X_j, Y_j, \wh Y_j) \Big] = \E[V^*_i(X_i)] ,
		\end{align}
		and the minimum is achieved by the estimators $(\psi^*_i,\ldots,\psi^*_n)$ defined in \eqref{eq:dp_policy}.
	\end{corollary}

	\subsection{An illustrative example}\label{sec:example}
	
	We now work out an example to illustrate how an optimal estimation strategy for dynamic inference works.
	Consider a situation where the observations, the quantities of interest, and the estimates all take binary values, i.e.\ $\sX=\sY=\wh\sY=\{0,1\}$.
	The observation-transition model is assumed to be stationary and deterministic, such that $X_i = 1-X_{i-1}$ if $\wh Y_{i-1}=0$, while $X_i=X_{i-1}$ if $\wh Y_{i-1}=1$,
	as depicted
	in Fig.~\ref{fig:fsm1}.
				The quantity-generation model is also stationary and is described by $P(Y=1|X=0)=0.1$ and $P(Y=1|X=1)=0.6$.
	The loss function neglects the observation and takes the form
	$
	\ell(x,y,\hat y) = \I\{y\neq \hat y\} .
	$

	\begin{figure}[h]
		\centering
		\includegraphics[scale = 0.46]{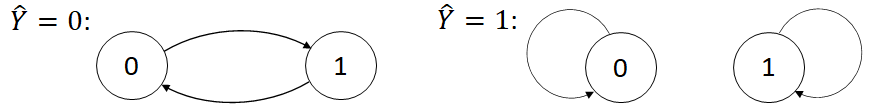}
		\caption{An example of stationary and deterministic observation-transition model.}
		\label{fig:fsm1}
	\end{figure}
	\begin{figure}[b]
	\centering
	\includegraphics[scale = 0.52]{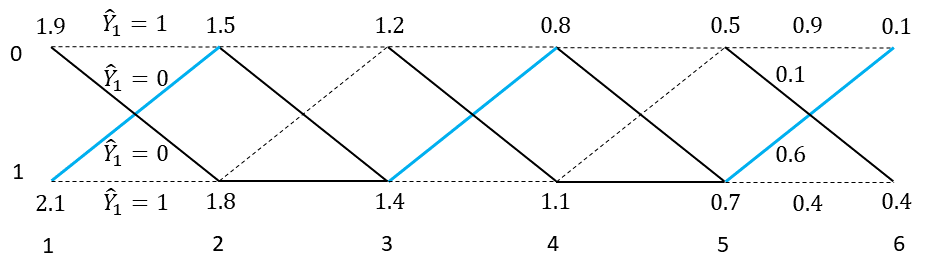}
	\caption{Unrolled observation-transition diagram for the dynamic inference example given in Section~\ref{sec:example}, with $n=6$. The value of $V^*$ and the optimal estimate at each observation are labeled. A solid branch indicates an optimal estimate, while a dashed branch indicates a non-optimal one. The three blue branches indicate the optimal estimates for dynamic inference 
		that are different from the optimal single-round estimates.}
	\label{fig:trelis_fsm1}
\end{figure}

	With this setup, the goal of dynamic inference is to minimize the expected number of wrong estimates during $n$ rounds of estimation.
	The optimal estimation strategy can be easily found by the dynamic programming procedure presented in Section~\ref{sec:sol_dp}.
	With $n=6$, the resulting values of the $V^*$ function are labeled at each observation in the unrolled observation-transition diagram in Fig.~\ref{fig:trelis_fsm1}.
	The optimal estimate at each observation is also labeled: a solid branch indicates an optimal estimate, while a dashed branch indicates a non-optimal one.
	We see that at each observation, the optimal estimate for dynamic inference can be different from that for the single-round estimation.
	For example, at $X_1=1$, $X_3=1$ and $X_5=1$, the optimal estimate for dynamic inference is $\wh Y = 0$, whereas the optimal estimate for the single-round estimation at these observations would be $\wh Y=1$ to minimize $\E[\I\{Y\neq\psi(X)\}]$.
	
	This example reveals a key difference between dynamic inference and the traditional statistical inference: in dynamic inference, the optimal estimate in each round strives to balance the loss-to-incur in that round and the loss-to-go from that round. 
	Consequently, an optimal estimation strategy may need to, at least occasionally, make non-optimal single-round estimates, in order to steer the future observations toward those with which the associated quantities of interest are easier to estimate or less costly if inaccurately estimated.

	\section{Two applications}\label{sec:app}
	Having formulated the dynamic inference problem and derived its solution, in this section we study its applications to real challenges.
	The two examples given below are simplistic, but they capture the essence of how dynamic inference can be used to model and solve various sequential and interactive estimation or prediction problems.

	\subsection{Stock trend prediction}\label{sec:app_sp}
	The first application is the prediction by big investors of the trend of stock, which could be the trend of the price of an individual stock or the index of a bundle of stocks.
	The trend, either {rising} or {falling}, statistically depends on some observable market signal, e.g.\ the supply-demand profile of the stocks under consideration.
	The prediction is sequentially made for several rounds, e.g.\ one round each day for that day's trend, each based on the past observed market signals. Once a prediction is made, it influences that day's investment decision and hence the supply-demand profile of the stocks under consideration, which will be reflected by the market signal in the next day and will further influence the next day's trend.
	
	Formally, for $n$ rounds of prediction, let $Y_i\in\{0,1\}$ be the trend in the $i$th round, which is to be predicted as $\wh Y_i$ based on the observable market signal $X_i$ in that round.
	To have the simplest observation model, we consider the situation where $X_i$'s take only two values $\{0,1\}$.
	The transition model of the next round's market signal given the signal and prediction in the current round can then be described by $P{\raisebox{-2pt}{$\scriptstyle X_{i}|X_{i-1}, \wh Y_{i-1}$}}$. 
	A stationary such model is shown in Fig.~\ref{fig:fsm_sp}.
	Additionally, the dependence of the trend on the market signal can be described by $P_{Y_i|X_i}$, e.g.\ $P(Y=1|X=0) = 0.4$ and $P(Y=1|X=1)=0.7$ where the trend positively correlates with the market signal.
	The loss function can be simply $\ell(x,y,\hat y) = \I\{y\neq \hat y\}$.
	\begin{figure}[h]
		\centering
		\includegraphics[scale = 0.56]{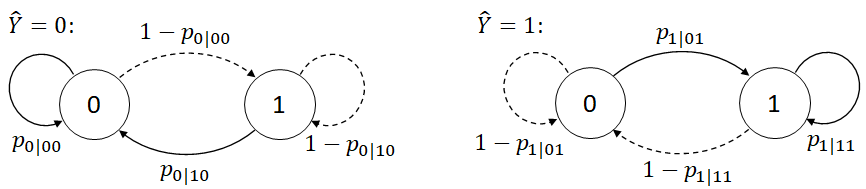}
		\caption{A stationary observation-transition model for stock trend prediction. The solid arrows represent the deterministic transition 
			$X_{i} = \wh Y_{i-1}$.}
		\label{fig:fsm_sp}
	\end{figure}
	
	\begin{figure*}
		\centering
		\includegraphics[scale = 0.52]{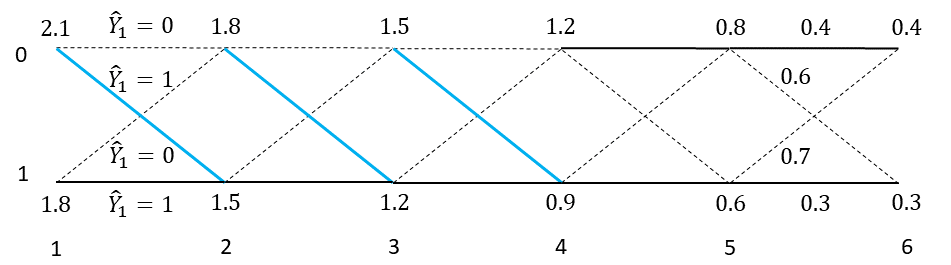}
		\caption{Unrolled observation-transition diagram for the stock trend prediction example, with $n=6$. The value of $V^*$ and the optimal prediction at each observation are labeled. A solid branch indicates an optimal prediction, while a dashed branch indicates a non-optimal one. The three blue branches indicate the optimal predictions for dynamic inference that are different from the optimal single-round predictions.}
		\label{fig:trelis_sp}
	\end{figure*}
	With all these elements specified, the problem is recognized as a dynamic inference problem similar to the example presented in Section~\ref{sec:example}, only with a more general observation-transition model. 
	Figure~\ref{fig:trelis_sp} shows the optimal estimation strategy of this example when the observation-transition model is deterministic, such that $X_i = \wh Y_{i-1}$.
	Similar to the example given in Section~\ref{sec:example}, we see from Fig.~\ref{fig:trelis_sp} that at $X_1=0$, $X_2=0$ and $X_3=0$, the optimal prediction for dynamic inference is different from the optimal single-round prediction.
	These predictions are made to steer the market signal to $1$, with which it is more certain that the trend will be rising according to $P_{Y|X}$, 
	hence smaller prediction error probability to occur.
	Only when $i\ge 4$, the optimal predictions coincide with the optimal single-round predictions, as the future prediction errors to accumulate weigh less toward the end of dynamic inference.

	\subsection{Vehicle behavior prediction}
	Another challenging problem that could be cast as dynamic inference is behavior prediction of vehicles on the road.
	For example, a desired feature of a self-driving system is to predict whether the following vehicle in the neighbor lane would yield if the ego vehicle initiates a cut-in to that lane, whenever there is a need for lane change. This task may be termed as \emph{yield prediction}, which can be sequential and interactive especially when the traffic is dense: the predicted intention determines the action to be taken by the ego vehicle, e.g.\ to turn on the blinker and initiate the cut-in when the following vehicle is predicted to yield, or not to cut-in and shoot for another gap when it is predicted not to yield; in response to the ego vehicle's behavior and according to the driving situation, the following vehicle would either slow down or accelerate, which can change the driving situation of the two vehicles and affect their subsequent behaviors; the interaction continues until the cut-in is completed due to a yield by some vehicle, or given up due to an opposite. 
	
	As in the stock trend prediction, we can formally define a dynamic inference problem for yield prediction.
	Suppose the prediction can be deconstructed into $n$ rounds. For the $i$th round, let $X_i$ denote the driving situation, which could be the positions and velocities of the vehicles under consideration, and let $Y_i\in\{{\mathsf{yield}}, {\mathsf{not\,\,yield}}\}$ represent the intention of the following vehicle.
	In the simplest setting, $X_i$ could be just the longitudinal bumper to bumper distance between the two vehicles, and the probabilistic model relating $Y_i$ to $X_i$ could be
	$
	P_{Y_i|X_i}({\mathsf{yield}} | x) = \sigma( \beta(x - d_{\rm c}))
	$
	where $\sigma(s)=\frac{1}{1+e^{-s}}$ is the logistic function, $\beta>0$ is a tunable parameter, and $d_{\rm c}$ is a critical distance that can be empirically determined.
	The observation-transition model in yield prediction is \emph{non-stationary} and depends on the design of the planner in the self-driving system. For example, when $X_i$ is small and $\wh Y_i = \mathsf{not\,\, yield}$, depending on the planner, $X_{i+1}$ could be either larger if the planner decides to increase the gap and still aims to cut-in, or smaller if the planner decides to slow down to shoot for another gap behind the following vehicle.
	Another feature of this problem is that the loss function should be \emph{contextual} and carefully designed.
	For example, when $Y_i = \mathsf{yield}$ and $\wh Y_i = \mathsf{not\,\, yield}$, the loss can be small and proportional to $X_i$, to moderately penalize a wasted chance for cut-in.
	On the other hand, when $Y_i = \mathsf{not\,\,yield}$ and $\wh Y_i = \mathsf{yield}$, the loss should be large especially when $X_i$ is small, to heavily penalize a wrong prediction that can lead to a dangerous situation.
	With all the elements specified, the problem can in principle be solved through dynamic programming as in Section~\ref{sec:sol_dp}, to minimize the overall prediction cost.

	To better illustrate the idea, three typical cases that can be encountered by the yield prediction are depicted in Fig.~\ref{fig:vbp}.
	In the first case, $X_1 > d_{\rm c}$, $Y_1 = \mathsf{yield}$, and $\wh Y_1 = \mathsf{yield}$. The ego vehicle initiates the cut-in by turning on the blinker, which results in a slow-down of the following vehicle, allowing the cut-in to be completed. There is only one round of prediction in this case.
	In the second case, $X_1 < d_{\rm c}$, $Y_1 = \mathsf{not \,\, yield}$, and $\wh Y_1 = \mathsf{yield}$. The ego vehicle initiates the cut-in by turning on the blinker, which results in an acceleration of the following vehicle, not allowing the cut-in to be completed, and leads to a dangerous driving situation. The ego vehicle then starts a second round of prediction under this situation.
	In the third case, $X_1 < d_{\rm c}$ and $\wh Y_1 = \mathsf{not \,\, yield}$. Since the ego vehicle predicts that the following vehicle will not yield if the blinker is on, it does not initiate a cut-in, and slows down to shoot for a gap behind the following vehicle. It then starts a second round of prediction, which can potentially be easier and less costly compared with the first round of prediction.
	By properly specifying the models and the loss function, a yield predictor designed under the framework of dynamic inference should enable the ego vehicle to drive in the first and the third case most of the time according to different driving situations, and avoid the behavior as in the second situation, unless it is deliberately designed to support aggressive cut-in.
	\begin{figure}[h]
		\centering
		\includegraphics[scale = 0.46]{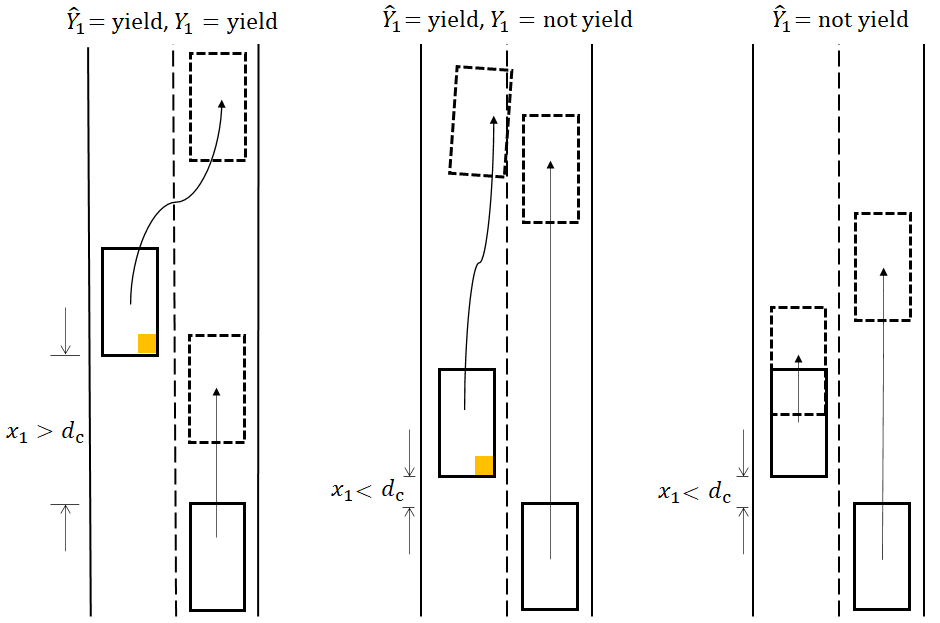}
		\caption{Three typical cases of interactive vehicle behaviors in yield prediction. 
		}
		\label{fig:vbp}
	\end{figure}

	\section{Learning for dynamic inference}\label{sec:learn_di}
	Solving the dynamic inference problem requires the knowledge of two important elements: the observation-transition model and the quantity-generation model.
	However, in many practically interesting situations, we may not have such knowledge. Instead, we either have a training dataset from which we can learn these models offline before doing inference, or we can learn them on-the-fly during inference if $Y_i$'s are revealed after each round of estimation.
	Such problems may be termed as \emph{learning for dynamic inference}, either offline or online.
	These problems can also be studied under a Bayesian framework, where the unknown models are assumed to be members of parametrized model families with certain priors, and the optimal learning rule that minimizes the expected inference loss can be mathematically derived \cite{Bayes_ldi}.  
	
	Perhaps more importantly, the problem of learning for dynamic inference can potentially serve as a meta problem for machine learning, such that almost all familiar learning problems can be cast as its special cases, examples including supervised learning, imitation learning, and reinforcement learning. 
	For instance, the \textit{offline} learning for dynamic inference can be viewed as an extension of the {behavior cloning} method in imitation learning \cite{IL_book,Grimes, Englert, BC_from_obs}, 
	in that it not only learns the demonstrator's action-generation model and state-transition model, but simultaneously learns a policy based on the learned models to minimize the overall contextual-aware imitation error. 
	As another instance, any loss function of the form $\bar\ell:\sX\times\wh\sY\rightarrow\R$ can be expressed in the integral form of \eqref{eq:def_bar_ell_comp}, with some contextual loss function $\ell:\sX\times\sY\times\wh\sY\rightarrow\R$ and probability transition kernel $P_{Y|X}$. The unknown quantity $Y$ that depends on $X$ can then be viewed as a latent variable of the loss function. With this view, any {reinforcement learning} problem \cite{RL_book} can be solved as an instance of \textit{online} learning for dynamic inference, where the quantities to be estimated are the latent variables of the loss function, and the quantity-generation model to be learned is the conditional distribution of the latent variable. 
	More detailed discussions on the connection between learning for dynamic inference and other learning problems are made in \cite{Bayes_ldi}.
	The study of dynamic inference and its learning extension thus help us gain deeper and unifying understandings of a broad spectrum of machine learning problems.
	
	\section{Multiple views of dynamic inference and related works}\label{sec:rel_work}
	The formulation of dynamic inference appears to be new, but it can be viewed from different angles, and is related to a variety of existing problems.
	\vspace{-10pt}
	\paragraph{Sequential interactive prediction as dynamic inference}
	The problems that can be most naturally formulated as dynamic inference are estimation or prediction problems in sequential and interactive settings. Traditionally some of those problems are formulated and studied using game theory \cite{beh_game_bk}.
	This type of problems become more common in recent years with widespread adoption of AI systems, e.g., they arise in behavior prediction of vehicles \cite{interac_pred_tomi19,evolvegraph20,ma_li_21},
	interactive recommendation with user feedback \cite{interac_rec_nlp}, and prediction in finance \cite{Ippoliti2017}. Dynamic inference provides a rigorous mathematical formulation of such problems, and provides an optimal solution to it.
	
	\vspace{-10pt}
	\paragraph{Dynamic inference as performative prediction}
	Dynamic inference also shares a similar spirit with a recent trend of research called performativity \cite{perf_lav19,perf_pred,perf_stateful}, where the problem is to deal with the tendency that the decision to make in optimization or prediction problems can change the underlying distribution the decision is made for. A method called repeated risk minimization is proposed to solve the performative prediction, either with \cite{perf_stateful} or without \cite{perf_pred} state transitions. The goal there is to minimize the loss in each single round of prediction, based on the distribution from the last round, and the hope is that such a method can reach a minimax equilibrium under certain conditions. On the contrary, dynamic inference aims to minimize the overall inference loss, and it explicitly considers multiple rounds of estimation and state transitions in these problems.
	
	\vspace{-10pt}
	\paragraph{Dynamic inference as imitation game}
	
	
	
	
	The formulation of dynamic inference can also be viewed as a game of imitation, where the learner drives a system with state, observes an action from a demonstrator at each encountered state, and tries to imitate the demonstrator's actions by minimizing the accumulated state-aware imitation error. When the underlying models are unknown, such a view can provide a rigorous formulation of imitation learning, both online and offline, and the optimal learning strategy is derived in \cite{Bayes_ldi}. In practice, this formulation has already been implicitly adopted by practitioners in imitation learning \cite{il_driving19}.
	
	\vspace{-10pt}
	\paragraph{MDP as dynamic inference}
	In this work, the solution to dynamic inference is derived by reformulating it to an MDP which can be solved by dynamic programming.
	Conversely, any MDP can be thought of as a dynamic inference problem, by viewing the loss function in an integral form that involves an unknown latent variable, as discussed in the previous section. The goal of MDP is then to estimate the latent variables by minimizing the overall estimation error.
	 This view also helps us to understand reinforcement learning, especially Bayesian reinforcement learning in the model-based form \cite{Strens,Poupart,BRL_book}, from a learning for dynamic inference perspective. The practical benefit of viewing MDP and reinforcement learning in this way would be an interesting research problem.

	\appendix
	
	
	
	
	
	
	\section*{Acknowledgement}
	The author is indebted to Peng Guan for many helpful discussions; the discussion with whom on imitation learning in early 2018 motivated this study.
	The author is grateful to Prof.\ Lav Varshney, for the detailed comments and many helpful suggestions, and for pointing out \cite{perf_pred} on performative prediction. The author also would like to thank Prof.\ Maxim Raginsky, for his encouragement in looking into dynamic aspects of statistical problems.
	
	\blfootnote{author email: xuaolin@gmail.com}

	\bibliography{DI_bib}

\end{document}